\PassOptionsToPackage{pdfpagelabels=false}{hyperref} 
\documentclass[english,a4paper,11pt]{article}
\usepackage{fullpage}

\usepackage{amsmath, amsxtra, amsfonts, amssymb, amstext, dsfont}

\usepackage{amsthm}

\usepackage{booktabs}
\usepackage{nicefrac}
\usepackage{xspace}
\usepackage[noadjust]{cite}
\usepackage{url}
\usepackage{graphics,color}
\usepackage[colorlinks]{hyperref}
\definecolor{linkblue}{rgb}{0.1,0.1,0.8}
\hypersetup{colorlinks=true,linkcolor=linkblue,filecolor=linkblue,urlcolor=linkblue,citecolor=linkblue}
\usepackage[algo2e,ruled,vlined,linesnumbered]{algorithm2e}
\newcommand{\assign}{\leftarrow}

\newtheorem{theorem}{Theorem}
\newtheorem{lemma}[theorem]{Lemma}

\newtheorem{corollary}[theorem]{Corollary}

\newcommand{\ignore}[1]{}
\newcommand{\todo}[1]{\textcolor{blue}{TODO: #1}}

\allowdisplaybreaks[1]

\newcommand{\N}{\mathbb{N}}
\newcommand{\R}{\mathbb{R}}

\renewcommand{\epsilon}{\varepsilon}
\newcommand{\eps}{\varepsilon}
\DeclareMathOperator{\Bin}{Bin}

\DeclareMathOperator{\aim}{aim}

\newcommand{\E}{E}

\newcommand{\onemax}{\textsc{OneMax}\xspace}

\newcommand{\OneMax}{\textsc{OneMax}\xspace}
\newcommand{\OM}{\textsc{Om}\xspace}

\newcommand{\oea}{$(1 + 1)$~EA\xspace}

\newcommand{\searchSpace}{\Omega}
\newcommand{\realnum}{\mathbb{R}}

\newcommand{\dInt}{d_{\mathrm{int}}}
\newcommand{\dRing}{d_{\mathrm{ring}}}

\begin{document} 
\title{The Right Mutation Strength\\ for Multi-Valued Decision Variables}
\author{Benjamin Doerr$^1$, Carola Doerr$^2$, Timo K{\"o}tzing$^3$}
\date{
$^1$Laboratoire d'Informatique (LIX), \'Ecole Polytechnique, Paris-Saclay, France\\
$^2$CNRS and LIP6, Sorbonne Universit\'es, UPMC Univ Paris 06, Paris, France\\
$^3$Hasso-Plattner-Institut, Potsdam, Germany
}

\maketitle

{\sloppy
\begin{abstract}
  The most common representation in evolutionary computation are bit strings. 	This is ideal to model binary decision variables, but less useful for variables taking more values. With very little theoretical work existing on how to use evolutionary algorithms for such optimization problems, 
		we study the run time of simple evolutionary algorithms on some OneMax-like functions defined over $\Omega = \{0, 1, \dots, r-1\}^n$. More precisely, we regard a variety of problem classes requesting the component-wise minimization of the distance to an unknown target vector $z \in \Omega$.  
		 
  For such problems we see a crucial difference in how we extend the standard-bit mutation operator to these multi-valued domains. While it is natural to select each position of the solution vector to be changed independently with probability $1/n$, there are various ways to then change such a position. If we change each selected position to a random value different from the original one, we obtain an expected run time of $\Theta(nr \log n)$. If we change each selected position by either $+1$ or $-1$ (random choice), the optimization time reduces to $\Theta(nr + n\log n)$. If we use a random mutation strength $i \in \{0,1,\ldots,r-1\}^n$ with probability inversely proportional to $i$ and change the selected position by either $+i$ or $-i$ (random choice), then the optimization time becomes $\Theta(n \log(r)(\log(n)+\log(r)))$, bringing down the dependence on $r$ from linear to polylogarithmic.
  
  One of our results depends on a new variant of the lower bounding multiplicative drift theorem.
	\end{abstract}

\section{Introduction}\label{sec:introduction}

In evolutionary computation, taking ideas both from computer science and biology, often search and optimization problems are modeled in a way that the solution candidates are fixed-length strings over the alphabet consisting of $0$ and~$1$. In other words, the search space $\Omega$ is chosen to be $\{0,1\}^n$ for some positive integer $n$. Such a representation of solution candidates is very suitable to model binary decision variables. For example, when searching for graph substructures like large cliques, (degree-constrained) spanning trees, or certain matchings, we can use binary decision variables describing whether a vertex or an edge is part of the solution or not. For these reasons, the bit string representation is the most prominent one in evolutionary computation. 

When a problem intrinsically consists of other types of decision variables, the algorithm designer has the choice to either work with a different representation (e.g., permutations in the traveling salesman problem) or to re-model the problem using a bit string representation. For an example for the latter, see, e.g.,~\cite{DoerrJ07gecco}, where the Eulerian cycle problem (asking for a permutation of the edges) was re-modeled as a matching problem. In general, such a re-modeling may not lead to an efficient or a natural approach, and it may be better to work with a representation different from bit strings. The traveling salesman problem is an example for such a situation.

While in this work we shall not deal with the difficulties of treating permutation search spaces in evolutionary computation, we shall try to extend our good understanding of the bit string representation to representations in which the decision variables can take more values than just zero and one. Consequently, we shall work with search spaces $\Omega = \{0, \ldots, r-1\}^n$. Such search spaces are a natural representation when each decision variable can take one out of $r$ values. Examples from the evolutionary computation literature include scheduling $n$ jobs on $r$ machines, which naturally leads to the search space $\{0, \ldots, r-1\}^n$, see Gunia~\cite{Gunia05}. However, also rooted trees lead to this type of representation: Since each vertex different from the root has a unique predecessor in the tree, a rooted tree on $n$ vertices can be represented via an element of $\{0, \ldots, n-1\}^{n-1}$. This was exploited in~\cite{ScharnowTW04} to design evolutionary algorithms for shortest-path problems.

An alternative representation would be to code each value in $\log r$ bits, leading to a search space of $\{0,1\}^{n \log r}$. However, this representation has the weakness that search points with similar fitness can be vastly different (the bit representations $10\ldots 0$ and $01\ldots 1$ code almost the same value, but are complementary); this trap-like behavior can lead to a very poor performance on some \OneMax functions (see Section~\ref{sec:OneMaxFunctions} for a formal defintion).

\subsection{Mutation Operators for Multi-Valued Search Spaces}

A first question, and our main focus in this work, is what mutation operators to use in such multi-valued search spaces. When there is no particular topology in the components $i \in [1..n] := \{1, \ldots, n\}$, that is, in the factors $[0..r-1]$, then the natural analogue of the standard-bit mutation operator is to select each component $i \in [1..n]$ independently and mutate the selected components by changing the current value to a random other value in $[0..r-1]$. This operator was used in~\cite{ScharnowTW04,Gunia05} as well as in the theoretical works~\cite{DoerrJS11,DoerrP12}.

When the decision values $0, 1, \ldots, r-1$ carry more meaning than just denoting alternatives without particular topology, then one may want to respect this in the mutation operator. We shall not discuss the most general set-up of a general distance matrix defined on the values $0, 1, \ldots, r-1$, but assume that they represent linearly ordered alternatives.
 
Given such a linear topology, several other mutation operators suggest itself. We shall always imitate the principle of standard-bit mutation that each component $i \in [1..n]$ is changed independently with probability $1/n$, so the only point of discussion is how such an elementary change looks like. The principle that mutation is a minimalistic change of the individual suggests to alter a selected component randomly by $+1$ or $-1$ (for a precise definition, including also a description of how to treat the boundary cases, see again Section~\ref{sec:preliminaries}). We say that this mutation operator has a \emph{mutation strength} equal to one. Naturally, a mutation strength of one carries the risk of being slow---it takes $r-1$ such elementary mutations to move one component from one boundary value, say $0$, to the other, say $r-1$. 

In this language, the previously discussed mutation operator changing a selected component to a new value chosen uniformly at random can (roughly) be described as having a mutation strength chosen uniformly at random from $[1..r-1]$. While this operator does not have the disadvantage of moving slowly through the search space, it does have the weakness that reaching a particular target is slow, even when already close to it.

Based on these (intuitive, but we shall make them precise later) observations, we propose an elementary mutation that takes a biased random choice of the mutation strength. We give more weight to small steps than the uniform operator, but do allow larger jumps with certain probability. More precisely, in each elementary mutation independently we choose the mutation strength randomly such that a jump of $+j$ or $-j$ occurs with probability inversely proportional to $j$ (and hence with probability $\Theta((j \log r)^{-1})$). This distribution was used in~\cite{Dit-Row-Weg-Woe:j:10} and is called \emph{harmonic distribution}, aiming at overcoming the two individual weaknesses of the two operators discussed before and, as we shall see, this does indeed work.

\subsection{Run time Analysis of Multi-Valued OneMax Functions}\label{sec:OneMaxFunctions}

To gain a more rigorous understanding of the working principles of the different mutations strengths, we conduct a mathematical run time analysis for simple evolutionary algorithms on multi-valued analogues of the OneMax test function. Comparable approaches have been very successful in the past in studying in isolation particular aspects of evolutionary computation, see, e.g.,~\cite{Jansen13}. Also, many observations first made in such simplistic settings have later been confirmed for more complicated algorithms (see, e.g.,~\cite{AugerD11}) or combinatorial optimization problems (see, e.g.,~\cite{NeumannW10}).

On bit strings, the classic OneMax test function is defined by $\OM: \{0,1\}^n \to \R; (x_1, \dots, x_n) \mapsto \sum_{i = 1}^n x_i$. Due to the obvious symmetry, for most evolutionary algorithms it makes no difference whether the target is to maximize or to minimize this function. For several reasons, among them the use of drift analysis, in this work it will be more convenient to always assume that our target is the minimization of the given objective function.

The obvious multi-valued analogue of this OneMax function is $\OM: \{0,1, \dots, r-1\}^n \to \R; x \mapsto \sum_{i=1}^n x_i$, however, a number of other functions can also be seen as multi-valued analogues. For example, we note that in the bit string setting we have $\OM(x) = H(x,(0,\dots,0))$, where $H(x,y):= |\{i \in [1..n] \mid x_i \neq y_i\}|$ denotes the Hamming distance between two bit strings $x$ and $y$. Defining $f_z: \{0,1\}^n\to \R; x \mapsto H(x,z)$ for all $z \in \{0,1\}^n$, we obtain a set of $2^n$ objective functions that all have an isomorphic fitness landscape. Taking this route to define multi-valued analogue of OneMax functions, we obtain the class of functions $f_z: \{0,1,\dots, r-1\}^n \mapsto \R; x \mapsto \sum_{i=1}^n |x_i-z_i|$ for all $z \in \{0,1,\dots,r-1\}^n$, again with $f_{(0,\dots,0)}$ being the OneMax function defined earlier. Note that these objective functions do not all have an isomorphic fitness landscape. The asymmetry with respect to the optimum $z$ can be overcome by replacing the classic distance $|x_i-z_i|$ in the reals by the distance modulo $r$ (ring distance), that is,  $\min\{x_i - (z_i-r), |x_i-z_i|, (z_i+r)-x_i\}$, creating yet another non-isomorphic fitness landscape. All results we show in the following hold for all these objective functions. 

As evolutionary algorithm to optimize these test functions, we study the (1+1) evolutionary algorithm (EA). This is arguably the most simple evolutionary algorithm, however, many results that could first only be shown for the \oea could later be extended to more complicated algorithms, making it an ideal instrument for a first study of a new subject. Naturally, to study mutation operators we prefer mutation-based EAs. For the different ways of setting the mutation strength, we conduct a mathematical run time analysis, that is, we prove bounds on the expected number of iterations the evolutionary algorithm needs to find the optimal solution. This \emph{optimization time} today is one of the most accepted performance measures for evolutionary algorithms.

\subsection{Previous Works and Our Results}

In particular for the situation that $r$ is large, one might be tempted to think that results from continuous optimization can be helpful. So far, we were not successful in this direction. A main difficulty is that in continuous optimization, usually the asymptotic rate of convergence is regarded. Hence, when operating with a fixed $r$ in our setting and re-scaling things into, say, $\{0,\frac 1r,\frac 2r,\dots,1\}^n$, then these results, due to their asymptotic nature, could become less meaningful. For this reason, the only work in the continuous domain that we found slightly resembling ours is by J\"agers\-k\"upper (see~\cite{Jagerskupper08} and the references therein), which regards continuous optimization with an a-priori fixed target precision. However, the fact that J\"agersk\"upper regards approximations with respect to the Euclidean norm (in other words, minimization of the sphere function) makes his results hard to compare to ours, which can be seen as minimization of the $1$-norm.

Coming back to the discrete domain, as said above, the vast majority of theoretical works on evolutionary computation work with a bit string representation. A notable exception is the work on finding shortest path trees (e.g.,~\cite{ScharnowTW04}); however, in this setting we have that the dimension and the number $r$ of values are not independent: one naturally has $r$ equal to the dimension, because each of the $n-1$ non-root vertices has to choose one of the $n-1$ other vertices as predecessor.

Therefore, we see only three previous works that are comparable to ours. The first two regard the optimization of linear functions via the \oea using mutation with uniform strength, that is, resetting a component to a random other value. The main result of~\cite{DoerrJS11} is that the known run time bound of $O(n \log n)$ on linear functions defined on bit strings remains valid for the search space $\{0,1,2\}^n$. This was extended and made more precise in~\cite{DoerrP12}, where for $r$-valued linear functions an upper bound of $(1+o(1))e(r-1)n\ln(n) + O(r^3 n \log\log n)$ was shown together with a $(1+o(1))n(r-1)\ln(n)$ lower bound. 

A third paper considers dynamically changing fitness functions~\cite{Koe-Lis-Wit:c:15}. They also consider OneMax functions with distance modulo $r$, using  $\pm 1$ mutation strength. In this setting the fitness function changed over time and the task was to track it as closely as possible, which the $\pm 1$ mutation strength can successfully do. Note that a seemingly similar work on the optimization of a dynamic variant of the maze function over larger alphabets~\cite{LissovoiW14} is less comparable to our work since there all non-optimal values of a decision variable contribute the same to the fitness function. 

Compared to these works, we only regard the easier static \onemax problem (note though that there are several ways to define multi-valued \onemax functions), but obtain tighter results also for larger values of $r$ and for three different mutation strengths. For the uniform mutation strength, we show a tight and precise $(1+o(1)) e (r-1) n \ln(n)$ run time estimate for all values of $r$ (Section~\ref{sec:uniform}). For the cautious $\pm 1$ mutation strength, the run time becomes $\Theta(n(r+\log n))$, that is, still (mostly) linear in $r$ (Section~\ref{sec:pm1}). The harmonic mutation strength overcomes this slowness and gives a run time of $\Theta(n \log(r)(\log(r) + \log(n)))$, which for most values of $r$ is significantly better than the previous bound (Section~\ref{sec:HarmonicStepSize}). 

All analyses rely on drift methods, for the lower bound for the case of uniform mutation strength we prove a variant of the multiplicative drift lower bound theorem~\cite{Witt13} that does not need the restriction that the process cannot go back to inferior search points (see Section~\ref{sec:lowerBoundDriftTheorem}).

\section{Algorithms and Problems}
\label{sec:preliminaries}

In this section we define the algorithms and problems considered in this paper. We let $[r]:=\{0,1,\ldots,r-1\}$ and $[1..r]:=\{1,2,\ldots,r\}$. For a given search space $\searchSpace$, a fitness function is a function $f: \searchSpace \rightarrow \R$. While a frequently analyzed search space is $\searchSpace = \{0,1\}^n$, we will consider in this paper $\searchSpace = [r]^n$.

We define the following two metrics on $[r]$, called \emph{interval-metric} and \emph{ring-metric}, respectively. The intuition is that the interval metric is the usual metric
induced by the metric on the natural numbers, while the ring metric connects the two endpoints of the interval (and, thus, forms a ring). Formally we have, for all $a,b \in [r]$,
\begin{eqnarray*}
\dInt(a,b) & = & |b-a|;\\
\dRing(a,b) & = & \min \{ |b-a|,|b-a+r|,|b-a-r|\}.
\end{eqnarray*}

We consider different \emph{step operators} $v: [r] \rightarrow [r]$ (possibly randomized). These step operators will later decide the update of a mutation in a given component. Thus we call, for any given $x \in [r]$, $d(x,v(x))$ the \emph{mutation strength}. We consider the following step operators.
\begin{itemize}
	\item The \emph{uniform step} operator chooses a different element from $[r]$ uniformly at random; thus we speak of a \emph{uniform mutation strength}.
	\item The \emph{$\pm 1$} operator chooses to either add or subtract $1$, each with probability $1/2$; this operator has a mutation strength of $1$. 
	\item The \emph{Harmonic} operator makes a jump of size $j \in [r]$ with probability proportional to $1/j$, choosing the direction uniformly at random; we call its mutation strength \emph{harmonic mutation strength}.
\end{itemize}
Note that, in the case of the ring-metric, all steps are implicitly considered with wrap-around. For the interval-metric, we consider all steps that overstep a boundary of the interval as invalid and discard this mutation as infeasible. Note that this somewhat arbitrary choice does not impact the results in this paper.

We consider the algorithms RLS and \oea as given by Algorithms~\ref{alg:rls} and~\ref{alg:ea}. Both algorithms sample an initial search point from $[r]^n$ uniformly at random. They then proceed in rounds, each of which consists of a mutation and a selection step. Throughout the whole optimization process the algorithms maintain a population size of one, and the individual in this population is always the most recently sampled best-so-far solution.
The two algorithms differ only in the \emph{mutation operation}. While the RLS makes a step in exactly one position (chosen uniformly at random), the \oea makes, in each position, a step with probability $1/n$. 

The fitness of the resulting search point $y$ is evaluated and in the \emph{selection step} the parent $x$ is replaced by its offspring $y$ if and only if the fitness of $y$ is at least as good as the one of~$x$. Since we consider minimization problems here, this is the case if $f(y) \leq f(x)$. Since we are interested in expected \emph{run times}, i.e., the expected number of rounds it takes until the algorithm evaluates for the first time a solution of minimal fitness, we do not specify a termination criterion. For the case of $r=2$, the two algorithms are exactly the classic Algorithms RLS and \oea, for all three given step operators (which then degenerate to the flip operator, which flips the given bit).

Note that the algorithms with the considered topologies are unbiased in the general sense of \cite{ABB} (introduced for $\{0,1\}^n$ by Lehre and Witt~\cite{LehreW12} and made specific for several combinatorial search spaces in~\cite{DoerrKLW13tcs}).

\begin{algorithm2e}	\textbf{Initialization:} 
	Sample $x \in [r]^n$ uniformly at random and query $f(x)$\;
  \textbf{Optimization:}
	\For{$t=1,2,3,\ldots$}{
		Choose $i \leq n$ uniformly at random\;
		\For{$j=1, \ldots, n$}{
			\lIf{j = i}{$y_j\assign v(x_j)$}
			\lElse{$y_j \assign x_j$}
		}
		Evaluate $f(y)$\;
		\lIf{$f(y)\leq f(x)$}{$x \assign y$}	
}\caption{RLS minimizing a function $f: {[r]}^n \rightarrow \R$ with a given step operator $v$.}
\label{alg:rls}
\end{algorithm2e}

\begin{algorithm2e}	\textbf{Initialization:} 
	Sample $x \in [r]^n$ uniformly at random and query $f(x)$\;
  \textbf{Optimization:}
	\For{$t=1,2,3,\ldots$}{
		\For{$i=1, \ldots, n$}
			{\label{line:mutEA}With probability $1/n$ set $y_i\assign v(x_i)$ and set $y_i \assign x_i$ otherwise\;}		Evaluate $f(y)$\;
		\lIf{$f(y)\leq f(x)$}{$x \assign y$}	
}\caption{The \oea minimizing a function $f:{[r]}^n \rightarrow \R$ with a given step operator $v$.}
\label{alg:ea}
\end{algorithm2e}
 
Let $d$ be either the interval- or the ring-metric and let $z \in [r]^n$. We can define a straightforward generalization of the \OneMax fitness function as
$$
\sum_{i = 1}^n d(x_i,z_i).
$$
Whenever we refer to an \emph{$r$-valued \OneMax function}, we mean any such function. We refer to $d$ as the \emph{metric of the \OneMax function} and to $z$ as the \emph{target of the \OneMax function}.

\section{Drift Analysis}
\label{sec:drift}

A central tool in many of our proofs is drift analysis, which comprises a number of tools to derive bounds on hitting times from bounds on the expected progress a process makes towards the target. Drift analysis was first used in evolutionary computation by He and Yao~\cite{HeY01} and is now, after a large number of subsequent works, probably the most powerful tool in run time analysis. We briefly collect here the tools that we use.

We phrase the following results in the language that we have some random process, either in the real numbers or in some other set $\Omega$, but then equipped with a potential function $g : \Omega \to \R$. We are mostly interested in the time the process (or its potential) needs to reach $0$.

\emph{Multiplicative drift} is the situation that the progress is proportional to the distance from the target. This quite common situation in run time analysis was first framed into a drift theorem, namely the following one, in~\cite{DoerrJW12}. A more direct proof of this results, that also gives large deviation bounds, was later given in~\cite{DoerrG13algo}.

\begin{theorem}[from \cite{DoerrJW12}]\label{thm:multidrift}
  Let $X^{(0)}, X^{(1)}, \dots$ be a random process taking values in $S := \{0\} \cup [s_{\min},\infty) \subseteq \R$. Assume that $X^{(0)} = s_0$ with probability one. Assume that there is a $\delta>0$ such that for all $t \ge 0$ and all $s \in S$ with $\Pr[X^{(t)} = s] > 0$ we have \[E[X^{(t+1)} | X^{(t)} = s] \le (1-\delta) s.\] Then $T := \min\{t \ge 0 \mid X^{(t)} = 0\}$ satisfies \[E[T] \le \frac{\ln(s_0/s_{\min})+1}{\delta}.\]
\end{theorem}

It is easy to see that the upper bound above cannot immediately be matched with a lower bound of similar order of magnitude. Hence it is no surprise that the only lower bound result for multiplicative drift, the following theorem by Witt~\cite{Witt13}, needs two additional assumptions, namely that the process does not move away from the target and that it does not too often make large jumps towards the target. We shall see later (Theorem~\ref{thm:newdrift}) that the first restriction can be removed under not too strong additional assumptions.

\begin{theorem}[from~\cite{Witt13}]\label{thm:multidriftlower} 
  Let $X^{(t)}, t = 0, 1, \ldots$ be random variables taking values in some finite set $S$ of positive numbers with $\min(S) = 1$. Let $X^{(0)} = s_0$ with probability one. Assume that for all $t \ge 0$, \[\Pr[X^{(t+1)} \le X^{(t)}] = 1.\]
  Let $s_{\aim} \ge 1$. Let $0 < \beta, \delta \le 1$ be such that for all $s > s_{\aim}$ and all $t \ge 0$ with $\Pr[X^{(t)} = s] > 0$, we have   
  \begin{align*}
	  E[&X^{(t)} - X^{(t+1)} \mid X^{(t)} = s] \le \delta s,\\
	  \Pr[&X^{(t)} - X^{(t+1)} \ge \beta s \mid X^{(t)} = s] \le \frac{\beta\delta}{\ln(s)}.
  \end{align*}
  Then $T := \min\{t \ge 0 \mid X^{(t)} \le s_{\aim}\}$ satisfies \[E[T] \ge \frac{\ln(s_0) - \ln(s_{\aim})}{\delta} \frac{1-\beta}{1+\beta}.\]
\end{theorem}

In situations in which the progress is not proportional to the distance, but only monotonically increasing with it, the following \emph{variable drift} theorem of Johannsen~\cite{Johannsen10} can lead to very good results. Another version of a variable drift theorem can be found in~\cite[Lemma~8.2]{Mit-Row-Can:j:09}.

\begin{theorem}[from~\cite{Johannsen10}]\label{thm:variabledrift}
  Let $X^{(t)}, t = 0, 1, \ldots$ be random variables taking values in some finite set $S$ of non-negative numbers. Assume $0 \in S$ and let $x_{\min} := \min(S \setminus \{0\})$. Let $X^{(0)} = s_0$ with probability one. Let $T := \min\{t \ge 0 \mid X^{(t)} = 0\}$. Suppose that there exists a continuous and monotonically increasing function $h: [x_{\min},s_0] \to \R_{>0}$ such that $E[X^{(t)} - X^{(t+1)} | X^{(t)}] \ge h(X^{(t)})$ holds for all $t < T$. Then \[E[T] \le \frac{x_{\min}}{h(x_{\min})} + \int_{x_{\min}}^{s_0} \frac 1 {h(x)} d x.\]
\end{theorem}

\section{Mutation Strength Chosen Uniformly at Random}
\label{sec:uniform}

In this section, we analyze the mutation operator with uniform mutation strength, that is, if the mutation operator chooses to change a position, it resets the current value to a different value chosen independently (for each position) and uniformly  at random. We shall prove the same results, tight apart from lower order terms, for all $r$-valued \onemax functions defined in Section~\ref{sec:preliminaries}. Let $f$ be one such objective function and let $z$ be its target. 

When regarding a single component $x_i$ of the solution vector, it seems that replacing a non-optimal $x_i$ by some $y_i$ that is closer to the target, but still different from it, gains us some fitness, but does not lead to a structural advantage (because we still need an elementary mutation that resets this value exactly to the target value $z_i$). This intuitive feeling is correct for RLS and not correct for the \oea.

\subsection{RLS with Uniform Mutation Strength}

For RLS, we turn the above intuition into the potential function $g : [r]^n \to \R; x \mapsto H(x,z) = |\{i \in [1..n] \mid x_i \neq z_i\}|$, the Hamming distance, which counts the number of non-optimal positions in the current solution $x$. We get both an upper and a lower bound on the drift in this potential which allow us to apply multiplicative drift theorems. From that we get the following result.

\begin{theorem}\label{thm:RLSUniformMutation}
  Let $f$ be any $r$-valued \onemax function with target $z \in [r]^n$. Then randomized local search (RLS) with uniform mutation strength has an optimization time $T$ satisfying \[E[T] = n (r-1) (\ln(n)+\Theta(1)).\] If $x_0$ denotes the random initial individual, then for all $x \in [r]^n$ we have \[E[T|x_0 = x] = n (r-1) H_{H(x,z)},\]
	where, for any positive integer $k$, we let $H_k:=\sum_{j=1}^k{1/j}$ denote the $k$-th Harmonic number.
\end{theorem}

\begin{proof}
Consider one iteration of RLS started with a current solution $x \neq z$. Let $y$ be the current solution after one iteration, that is, the value of $x$ after mutation and selection. We observe that $g(y) = g(x) - 1$ if and only if the mutation operator selects a non-optimal position $i$ of $x$ (this happens with probability $g(x)/n$) and then replaces $x_i$ by $z_i$ (this happens with probability $1/(r-1)$). In all other cases, we have $g(y) = g(x)$, though not necessarily $y=x$. Consequently, the expected progress with respect to $g$ in this iteration is
\begin{equation}
  g(x) - E[g(y)] = \frac{g(x)}{n (r-1)}.\label{eq:rlsunif}
\end{equation}
 Let us denote by $T_{x_0}$ the run time of RLS conditional on the initial search point being $x_0$. Then the multiplicative drift theorem (Theorem~\ref{thm:multidrift}) gives an upper bound of \[E[T_{x_0}] \le n (r-1) (\ln(g(x_0))+1).\]

Similarly, the assumptions of the multiplicative drift theorem for lower bounds (Theorem~\ref{thm:multidriftlower}) are satisfied with $s_{\aim} =  \ln n$ and $\beta= 1/ \ln n$. Consequently, assuming $g(x_0) = \exp(\omega(\ln\ln n))$ in the second estimate, we obtain 
\begin{align*}
  E[T_{x_0}] &\ge n (r-1) (\ln(g(x_0)) - \ln\ln n)(1-2/\ln(n)) \\
  & = n (r-1) \ln(g(x_0)) (1-o(1)).
\end{align*}

In the above analysis we used multiplicative drift with the Hamming distance because this in a generic manner gave a very strong result. We also used a drift approach to ease the comparison with the other results we will obtain, also via drift analysis. For this particular problem, also a very problem-specific approach can be used, which gives an even sharper result. Consider a run of RLS starting with a search point $x_0$. For $i \in [0..g(x_0)]$, let $T_i$ denote the first iteration after which $g(x) \le i$, where $T_{g(x_0)} = 0$. Then equation~(\ref{eq:rlsunif})	 shows that $E[T_{i-1} - T_i] = \frac{n (r-1)}{i}$ for all $i \in [1..g(x_0)]$. Consequently, 
\begin{align*}
E[T_{x_0}] &= E\bigg[\sum_{i = 1}^{g(x_0)} (T_{i-1} - T_i)\bigg]\\
	& = \sum_{i = 1}^{g(x_0)} E[T_{i-1} - T_i]\\
	& = n (r-1) \sum_{i=1}^{g(x_0)} \frac 1i\\
	& = n (r-1) H_{g(x_0)},
\end{align*}
 where for all $k \in \N$, $H_k := \sum_{i=1}^k (1/i)$ is the $k$th Harmonic number. The harmonic number is well understood, e.g., we have $H_k = \ln(k) + \gamma + O(1/k)$ with $\gamma = 0.5772...$ being the Euler-Mascheroni constant and we have the non-asymptotic bounds $\ln(k) \le H_k \le \ln(k)+1$, which gives \[n (r-1) \ln(g(x_0)) \le E[T_{x_0}] \le n (r-1) (\ln(g(x_0))+1).\]

By the law of total probability, the expected run time of RLS (with the usual random initialization) is $E[T] = n (r-1) E[H_{g(x_0)}]$. The expected potential of the random initial search point is $E[g(x_0)] = n (1 - 1/r)$. By a Chernoff bound (e.g., Theorem~1.11 in~\cite{Doerr11bookchapter}), we see that $\Pr[|g(x_0) - E[g(x_0)]| \ge \sqrt{n \ln n}] \le 2n^{-2}$. Hence $E[H_{g(x_0)}] = H_{E[g(x_0)]} \pm \Theta(\sqrt{\ln(n)/n}) \pm \Theta(\ln(n)/n^2)$. The first error term could be further reduced by arguments as used in~\cite{DoerrD14}, where for the case $r=2$ a run time bound of $E[T] = n H_{n/2} - 1/2 \pm o(1)$ was shown. We do not detail this idea any further and are content with summarizing the above in the following result, which in particular shows that the Hamming distance very precisely describes the quality of a search point.
\end{proof}

\subsection{The (1+1) EA with Uniform Mutation Strength}

We now consider the same run time analysis problem for the \oea, that is, instead of selecting a single random entry of the solution vector and applying an elementary mutation to it, we select each entry independently with probability $1/n$ and mutate all selected entries. Our main result is the following.
\begin{theorem}\label{thm:eauniform}
  For any $r$-valued \onemax function, the \oea with uniform mutation strength has an expected optimization time of \[E[T] = e (r-1) n \ln(n) + o(n r \log n).\]
\end{theorem}

As we will see, since several entries can be changed in one mutation step, the optimization process now significantly differs from the RLS process. This has two important consequences. First, while for the RLS process the Hamming distance of the current search point precisely determined the expected remaining optimization time, this is not true anymore for the \oea. This can be seen (with some mild calculations which we omit here) from the search points $x = (r,0,\dots,0)$ and $y = (1,0,\dots,0)$ and the fitness function $f$ defined by $f(x) = \sum_{i=1}^n x_0$. 

The second, worse, consequence is that the Hamming distance does not lead to a positive drift from each search point. Consider again $x = (r,0,\dots,0)$ and $f$ as above. Denote by $x'$ the search point after one mutation-selection cycle started with $x$. Let $g$ be the Hamming distance to the optimum $x^* = (0,\dots,0)$ of $f$. Then for $r \ge 5$, the drift from the search point $x$ satisfies $E[g(x) - g(x')] \le -(1\pm o(1)) \frac{r-4}{2e (r-1) n} < 0$. Indeed, we have $g(x')=0$, that is, $g(x)-g(x')=1$, with probability $(1-1/n)^{n-1}(1/n)(1/(r-1)) = (1\pm o(1)) \frac{1}{e (r-1) n}$. This is the only event that gives a positive drift. On the other hand, with probability at least $(1-1/n)^{n-2}(n-1)(1/n^2)(1+2+\dots+(r-2))/(r-1)^2 = (1\pm o(1)) \frac{r-2}{2e (r-1) n}$, the mutation operator touches exactly the first and one other entry of $x$ and does so in a way that the first entry does not become zero and the second entry remains small enough for $x'$ to be accepted. This event leads to a drift of $-1$, showing the claim.

For these reasons, we resort to the actual fitness as potential function in our upper bound proof. It is clear that the fitness also is not a perfect measure for the remaining optimization time (compare, e.g., the search points $(2,0,\dots,0)$ and $(1,1,0,\dots,0)$), but naturally we have a positive drift from each non-optimal search point, which we shall exploit via the variable drift theorem. For the lower bound, a worsening of the Hamming distance in the optimization process is less of a problem, since we only need an upper bound for the drift. Hence for the lower bound, we can use multiplicative drift with $g$ again. However, since the process may move backwards occasionally, we cannot apply Witt's lower bound drift theorem (Theorem~\ref{thm:multidriftlower}), but have to prove a variant of it that does not require that the process only moves forward. This lower bound theorem for multiplicative drift might be of interest beyond this work. 

\subsubsection{An Upper Bound for the Run Time}

\begin{theorem}\label{thm:eauniformupper}
  For any $r$-valued \onemax function $f$, the \oea with uniform mutation strength has an expected optimization time of 
\begin{align*}
E[T] & \le e (r-1) n \ln(n) + (2 + \ln(2)) e(r-1)n\\
 & = e (r-1) n \ln(n) + O(rn).
\end{align*}
\end{theorem}

\begin{proof}
Let $z$ be the optimum of $f$. Then $f$ can be written as $f(x) = \sum_{i=1}^n d(x_i,z_i)$, where $d$ is one of the distance measures on $[r]$ that were described in Section~\ref{sec:preliminaries}. Let $x$ be a fixed search point and $y$ be the result of applying one mutation and selection step to $x$. We use the short-hand $d_i := d(x_i,z_i)$. We first show that 
\begin{equation}\label{eq:disquare}
  \Delta := f(x) - E[f(y)] \ge \frac{1}{2e(r-1)n} \sum_{i=1}^n d_i (d_i+1).
\end{equation}
 Indeed, $f(x) - f(y)$ is always non-negative. Consequently, it suffices to point out events that lead to the claimed drift. With probability $(1-(1/n))^{n-1} \ge (1/e)$, the mutation operator changes exactly one position of $x$. This position then is uniformly distributed in $[1..n]$. Conditional on this position being $i$, we have $\Delta \ge \sum_{\delta = 1}^{d_i} \delta/(r-1) = \frac{d_i (d_i+1)}{2(r-1)}$, where the first inequality uses the fact that all our fitness functions are of the type that if there is a value $x_i \in [r]$ with $d(x_i,z_i) = k$, then for each $j \in [0..k-1]$ there is at least one value $y_i \in [r]$ such that $d(y_i,z_i) = j$. This shows~(\ref{eq:disquare}). 
 
For any $d \ge 1$, we have $d(d+1) \ge 2d$ and $d(d+1) \ge d^2$. Also, the mapping $d \mapsto d^2$ is convex. Consequently, we have $\Delta \ge \frac{1}{2e(r-1)n^2} f(x)^2$ and $\Delta \ge \frac 1 {e(r-1)n} f(x)$, that is, $\Delta \ge \max\{\frac{1}{2e(r-1)n^2} f(x)^2,  \frac 1 {e(r-1)n} f(x)\}$. To this drift expression, we apply Johannsen's~\cite{Johannsen10} variable drift theorem (Theorem~\ref{thm:variabledrift}). Let $S = [0..(r-1)n]$. Let $h: \R_{> 0} \to \R_{> 0}$ be defined by $h(s) = \frac{1}{2e(r-1)n^2}s^2$ for $s \ge 2 n$ and $h(s) = \frac 1 {e(r-1)n} s$ for $s < 2 n$. Then $h$ is a continuous increasing function satisfying $\Delta \ge h(f(x))$. Consider the process $X_0, X_1, \ldots$ with $X_t$ describing the fitness after the $t$th iteration. Given that we start with a fitness of $X_0$, Johannsen's drift theorem gives 
\begin{align*}
& E[T] \le \frac 1 {h(1)} + \int_1^{X_0} \frac 1 {h(s)} ds\\
 &  = e(r-1)n + \int_{2 n}^{X_0} 2e(r-1)\frac{n^2}{s^{2}} ds + \int_{1}^{2 n} e(r-1)\frac{n}{s} ds \\
 & \le e(r\!-\!1)n + 2e(r\!-\!1)n^2\!\left(\frac{1}{2 n} \!-\! \frac{1}{X_0}\right) + e(r\!-\!1)n \ln(2 n) \\
 & \le e(r-1)n \ln(n) + (1+ 1 +  \ln(2)) e(r-1)n. \qedhere
\end{align*}
\end{proof}

We may remark that the drift estimate above is pessimistic in that it applies to all $r$-valued \onemax functions. For an $r$-valued \onemax function using the ring metric or one having the optimum close to $(r/2,\dots,r/2)$, we typically have two different bit values in each positive distance from $z_i$. In this case, the drift stemming from exactly position $i$ being selected for mutation is $\Delta \ge d_i/(r-1) + \sum_{\delta = 1}^{d_i-1} 2\delta/(r-1) = \frac{d_i^2}{r-1}$, that is, nearly twice the value we computed above. The fact that in the following subsection we prove a lower bound matching the above upper bound for all $r$-valued \onemax functions shows that this, almost twice as high, drift has an insignificant influence on the run time. 

\subsubsection{A Lower Bound for the Run Time}\label{sec:lowerBoundDriftTheorem}

In this section, we write $(q)_+ := \max\{q,0\}$ for any $q \in \R$. We aim at proving a lower bound, again via drift analysis, that is, via transforming an upper bound on the expected progress (with respect to a suitable potential function) into a lower bound on the expected run time. Since we only need an upper bound on the progress, we can again (as in the RLS analysis) work with the Hamming distance $g(x) = H(x,z)$ to the optimum $z$ as potential and, in the upper estimate of the drift, ignore the fact that this potential may increase. The advantage of working with the Hamming distance is that the drift computation is easy and we observe multiplicative drift, which is usually convenient to work with.

We have to overcome one difficulty, though, and this is that the only known lower bound theorem for multiplicative drift (Theorem~\ref{thm:multidriftlower}) requires that the process does not move away from the target, in other words, that the $g$-value is non-increasing with probability one. As discussed above, we do not have this property when using the Hamming distance as potential in a run of the \oea. We solve this problem by deriving from Theorem~\ref{thm:multidriftlower} a drift theorem (Theorem~\ref{thm:newdrift} below) that gives lower bounds also for processes that may move away from the optimum. Compared to Theorem~\ref{thm:multidriftlower}, we need the stronger assumptions (i) that we have a Markov process and (ii) that we have bounds not only for the drift $g(X^{(t)}) - g(X^{(t+1)})$ or the positive part $(g(X^{(t)}) - g(X^{(t+1)}))_+$ of it, but also for the positive progress $(s - g^{(t+1)})_+$ with respect to any reference point $s \le g(X^{(t)})$. This latter condition is very natural. In simple words, it just means that we cannot profit from going back to a worse (in terms of the potential) state of the Markov chain. 

A second advantage of these stronger conditions (besides allowing the analysis of non-decreasing processes) is that we can easily ignore an initial segment of the process (see Corollary~\ref{cor:newdrift}). This is helpful when we encounter a larger drift in the early stages of the process. This phenomenon is often observed, e.g.,  in~Lemma~6.7 of~\cite{Witt13}. Previous works, e.g.,~\cite{Witt13}, solved the problem of a larger drift in the early stage of the process by manually cutting off this phase. This requires again a decreasing process (or conditioning on not returning to the region that has been cut off) and an extra argument of the type that the process with high probability reaches a search point with potential in $[\tilde s_0, 2\tilde s_0]$ for a suitable $\tilde s_0$. So it is safe to say that Corollary~\ref{cor:newdrift} is a convenient way to overcome these difficulties.

We start by proving our new drift results, then compute that the Hamming distance to the optimum satisfies the assumptions of our drift results, and finally state and prove the precise lower bound.

\begin{theorem}\textsc{(multiplicative drift, lower bound, non-decreasing process)}\label{thm:newdrift} 
  Let $X^{(t)}, t = 0, 1, \ldots$ be a Markov process taking values in some set $\Omega$. Let $S \subset \R$ be a finite set of positive numbers with $\min(S) = 1$. Let $g : \Omega \to S$. Let $g(X^{(0)}) = s_0$ with probability one. Let $s_{\aim} \ge 1$. Let \[T := \min\{t \ge 0 \mid g(X^{(t)}) \le s_{\aim}\}\] be the random variable describing the first point in time for which $g(X^{(t)}) \le s_{\aim}$. 
  
  Let $0 < \beta, \delta \le 1$ be such that for all $\omega \in \Omega$, all $s_{\aim} < s \le g(\omega)$, and all $t \ge 0$ with $\Pr[X^{(t)} = \omega] > 0$, we have   
  \begin{align*}
	  E[&(s - g(X^{(t+1)}))_+ \mid X^{(t)} = \omega] \le \delta s,\\
	  \Pr[&s - g(X^{(t+1)}) \ge \beta s \mid X^{(t)} = \omega] \le \frac{\beta\delta}{\ln(s)}.
  \end{align*}
  Then \[E[T] \ge \frac{\ln(s_0) - \ln(s_{\aim})}{\delta} \frac{1-\beta}{1+\beta} \ge \frac{\ln(s_0) - \ln(s_{\aim})}{\delta} (1-2\beta).\]
\end{theorem}

The proof follows from an application of Witt's drift theorem (Theorem~\ref{thm:multidriftlower}) to the random process $Y^{(t)} := \min\{g(X^{(\tau)}) \mid \tau \in [0..t]\}$.

\begin{proof}
  We define a second random process by $Y^{(t)} := \min\{g(X^{(\tau)}) \mid \tau \in [0..t]\}$. By definition, $Y$ takes values in $S$ and $Y$ is decreasing, that is, we have $Y^{(t+1)} \le Y^{(t)}$ with probability one for all $t \le 0$. Trivially, we have $Y^{(0)} = g(X^{(0)}) = s_0$. Let $T_Y := \min\{t \ge 0 \mid Y^{(t)} \le s_{\aim}\}$ be the first time this new process reaches or goes below $s_{\aim}$. Clearly, $T_Y = T$. 
  
  Let $\beta, \delta$ as in the theorem. Let $s_{\aim} < s$ and $t \ge 0$ such that $\Pr[Y^{(t)} = s] > 0$. Observe that when $Y^{(t)} = s$, then $Y^{(t)} - Y^{(t+1)} = s - \min\{s,g(X^{(t+1)})\} = (s - g(X^{(t+1)}))_+$. Let $A^Y_s$ be the event that $Y^{(t)} = s$ and let $B^X_{\omega}$ be the event that $X^{(t)} = \omega$. Using the fact that $X$ is a Markov process, we compute
  \begin{align*}
	& E[Y^{(t)} - Y^{(t+1)} \mid A^Y_s] \\
	& = \sum_{\omega : g(\omega) \ge s} \Pr[A^Y_s \mid B] E[(s - g(X^{(t+1)}))_+ \mid A^Y_s, B^X_{\omega}]\\
	& = \sum_{\omega : g(\omega) \ge s} \Pr[B^X_{\omega} \mid A^Y_s] E[(s - g(X^{(t+1)}))_+ \mid B^X_{\omega}]\\
  & \le \sum_{\omega : g(\omega) \ge s} \Pr[B^X_{\omega} \mid A^Y_s]  \delta s = \delta s
  \end{align*}
  and
  \begin{align*}
	\Pr[Y^{(t)} &- Y^{(t+1)} \ge \beta s \mid A^Y_s]\\
	& = \sum_{\omega : g(\omega) \ge s} \Pr[B^X_{\omega} \mid A^Y_s] \Pr[s - X^{(t+1)} \ge \beta s \mid A^Y_s, B^X_{\omega}]\\
	& = \sum_{\omega : g(\omega) \ge s} \Pr[B^X_{\omega} \mid A^Y_s] \Pr[s - X^{(t+1)} \ge \beta s \mid B^X_{\omega}]\\
	& \le  \sum_{\omega : g(\omega) \ge s} \Pr[B^X_{\omega} \mid A^Y_s]  \frac{\beta\delta}{\ln(s)} = \frac{\beta\delta}{\ln(s)}.
	\end{align*}
  Consequently, $Y$ satisfies the assumptions of the multiplicative lower bound theorem (Theorem~\ref{thm:multidriftlower}). Hence $E[T] = E[T_Y] \ge \frac{\ln(s_0) - \ln(s_{\aim})}{\delta} \frac{1-\beta}{1+\beta}$. Elementary algebra shows $(1-\beta) \ge (1-2\beta)(1+\beta)$, which gives the second, more convenient lower bound.
\end{proof}

\begin{corollary}\label{cor:newdrift}
  Assume that the assumptions of Theorem~\ref{thm:newdrift} are satisfied, however with $\delta$ replaced by $\delta(s)$ for some function $\delta : S \to (0,1]$. Then for any $s_{\aim} < \tilde s_0 \le s_0$, we have \[E[T] \ge \frac{\ln(\tilde s_0) - \ln(s_{\aim})}{\delta_{\max}(\tilde s_0)} (1-2\beta),\] where $\delta_{\max}(\tilde s_0) := \max\{\delta(s) \mid s_{\aim} < s \le \tilde s_0\}$. 
\end{corollary}

\begin{proof}
  Let $\tilde S := S \cap [0,\tilde s_0]$. Let $\tilde g: \Omega \to \tilde S; \omega \mapsto \min\{\tilde s_0,g(\omega)\}$. Let $\omega \in \Omega$, $s_{\aim} < s \le \tilde g(\omega)$, and $t$ be such that $\Pr[X^{(t)} = \omega] > 0$. Then 
\begin{align*}
   E[(s-\tilde g&(X^{(t+1)}))_+ \mid X^{(t)} = \omega]\\
  & = E[(s-g(X^{(t+1)}))_+ \mid X^{(t)} = \omega]\\
  & \le \delta(s) s \le \delta_{\max}(\tilde s_0) s
\end{align*}  
   by the assumptions of Theorem~\ref{thm:newdrift} and $s \le \tilde s_0$. Similarly, 
\begin{align*}
  \Pr[s - \tilde g&(X^{(t+1)})  \ge \beta s \mid X^{(t)} = \omega]\\
  & = \Pr[s - g(X^{(t+1)}) \ge \beta s \mid X^{(t)} = \omega]\\
  & \le \frac{\beta\delta(s)}{\ln(s)}
   \le \frac{\beta\delta_{\max}(\tilde s_0)}{\ln(s)}.
\end{align*}
  
   Hence we may apply Theorem~\ref{thm:newdrift} to $(\tilde S, \tilde s_0, \tilde g, \delta_{\max}(\tilde s_0))$ instead of $(S,s_0,g,\delta)$ and obtain the claimed bound.
\end{proof}

\begin{lemma}\label{lem:delta}  Let $f$ be an $r$-valued $\onemax$ function with optimum $z$. Let $x \in [r]^n$ and $y$ be the outcome of applying mutation and selection to $x$. Let $s^+ := H(x,z)$ and $s \le s^+$. Then \[E[(s - H(y,z))_+] \le \frac{s}{e(r-1)n} \bigg(\frac 1{1-1/n} + 3e\frac{s+1}{(r-1)n}\bigg).\]
\end{lemma}

\begin{proof}
Let $u$ be the outcome of mutating $x$ with uniform mutation strength and $y$ be the result of applying selection (with respect to $f$) to $x$ and $u$. 

We consider first the case that $s = s^+$. With probability $(1 - (1/n))^{n-1}$, $u$ and $x$ differ in exactly one position. 
Conditional on this, $E[(s - H(y,z))_+] = s / (r-1)n$. The only other event in which possibly $s > H(y,z)$ is that $u$ and $x$ differ in at least two positions $i$ and $j$ such that $u_i=z_i$ and $u_j = z_j$. The probability for this to happen is $1 - (1-1/(r-1)n)^s - (s/(r-1)n)(1-1/(r-1)n)^{s-1} \le 1 - (1-s/(r-1)n) - (s/(r-1)n)(1-(s-1)/(r-1)n) = s(s-1) / (r-1)^2 n^2$. In this case, we can estimate $(s-H(y,z))_+$ from above by the number of non-correct positions that are touched by the mutation operator, which in this case is $\Bin(s,1/n)$ conditional on being at least two, which again is at most $3$.
Consequently, in the case that $s = s^+$, we have $E[(s-H(y,z)_+] \le (1 - (1/n))^{n-1} s / (r-1)n + 3s(s-1)/(r-1)^2 n^2 \le (s/(r-1)n) (1/e(1-1/n) + 3(s-1)/(r-1)n)$.

Let now $s \le s^+ - 1$. Let $Z := |\{i \in [1..n] \mid x_i \neq z_i \neq u_i\}| \in [0..s^+]$ be the number of positions that are incorrect in both $x$ and $u$. Clearly, $H(y,z)$ stochastically dominates $Z$, which we write as $H(x,z) \succeq Z$. Let $Z'$ be defined analogous to $Z$ but for an original search point $x'$ with $H(x',z) = s+1 \le H(x,z)$. Then, clearly, $Z \succeq Z'$. Consequently, $s - H(y,z) \preceq s-Z \preceq s - Z'$, and consequently, $(s-H(y,z))_+ \preceq (s-Z')_+$ and $E[(s-H(y,z)_+] \le E[(s-Z')_+]$. The only way to get a positive value for $s - Z'$ is that at least two incorrect positions of $x'$ are changed to their correct value in the mutation offspring. Analogous to the previous paragraph, the probability for this to happen is $1 - (1-1/(r-1)n)^{s+1} - ((s+1)/(r-1)n)(1-1/(r-1)n)^{s} \le (s+1)s / (r-1)^2 n^2$. In this case, we can estimate $(s-Z')_+$ from above by the number of incorrect positions that are touched by the mutation operator (conditional on being at least two) minus one, which is at most $2$. We conclude $E[(s-H(y,z))_+] \le E[(s-Z')_+] \le 2(s+1)s/(r-1)^2 n^2$. 

Putting the two cases together, we see that we always have $E[(s-H(y,z))_+] \le (s/(r-1)n) (1/e(1-1/n) + 3(s+1)/(r-1)n)$.
\end{proof}

\begin{lemma}\label{lem:beta}  Let $f$ be an $r$-valued $\onemax$ function with optimum $z$. Let $s_{\aim} = \ln(n)^3$ and $\beta = 1 / \ln(n)$. Let $x \in [r]^n$ with $H(x,z) > s_{\aim}$. Let $s_{\aim} < s \le H(x,z)$. Let $y$ be the outcome of applying mutation and selection to $x$. Then $\Pr[s - H(y,z) \ge \beta s] \le \frac{1}{r-1} 2^{-\ln(n)^2}$ if $n \ge 11$.
\end{lemma}

\begin{proof}
We have that $\Pr[s - H(y,z) \ge \beta s] \le \Pr[H(x,z) - H(y,z) \ge \beta s_{\aim}]$. The latter is at most the probability that at least $\beta s_{\aim} = \ln(n)^2$ positions of $x$ flip to a particular value (namely the one given by $z$) in one mutation step. Since the expected number of positions flipping to the correct value is at most $1/(r-1)$, a strong multiplicative Chernoff bound (e.g., Cor.~1.10(b) in~\cite{Doerr11bookchapter}) shows that this number is greater than $\ln(n)^2$ with probability at most $(e / \ln(n)^2(r-1))^{\ln(n)^2} \le \frac{1}{r-1} 2^{-\ln(n)^2}$ for $n \ge 10.29 \approx \exp(\sqrt{2e})$.
\end{proof}

We are now ready to give the main result of this section.

\begin{theorem}\label{thm:eauniformlower}
  For any $r$-valued \onemax function, the \oea with uniform mutation strength has an expected optimization time of 
  \begin{align*}
  E[T] &\ge e(r-1)n \,(\ln(n) - 6\ln\ln(n)) \, (1 - O(1/\ln(n))) \\
  	&\ge e(r-1)n\ln(n) - O((r-1)n\ln\ln(n)).
  \end{align*}
\end{theorem}

\begin{proof}  Let $n$ be sufficiently large. Let $\Omega = [r]^n$ and $f: \Omega \to \R$ an $r$-valued \onemax function with optimum $z$. Let $s_{\aim} = \ln(n)^3$ and $\beta = 1 / \ln(n)$. For all $s_{\aim} < s \le n$, let $\delta(s) := (1/e(r-1)n) (1/(1-1/n) + 3e(s+1)/(r-1)n)$. Consider a run of the \oea optimizing $f$ initialized with a random search point $X^{(0)}$. We have $E[H(X^{(0)},z)] = n(1 - 1/r)$. Consequently, we have $H(X^{(0)},z) \ge n/3$ with probability $1 - \exp(-\Omega(n))$. In the following, we thus assume that $X^{(0)}$ is a fixed initial search point with some $H(\cdot,z)$ value of at least $n/3$. Denote by $X^{(t)}$ the search point building the one-element population of this EA after the $t$-th iteration. Let $g : \Omega \to \N; x \mapsto H(x,z)$. By Lemma~\ref{lem:delta} and~\ref{lem:beta}, the following conditions are satisfied for all $\omega \in \Omega$, all $s_{\aim} < s \le g(\omega)$, and all $t \ge 0$ with $\Pr[X^{(t)} = \omega] > 0$.
  \begin{align*}
	  E[&(s - g(X^{(t+1)}))_+ \mid X^{(t)} = \omega] \le \delta(s) s.\\
	  \Pr[&s - g(X^{(t+1)}) \ge \beta s \mid X^{(t)} = \omega] \le \frac{\beta\delta(s)}{\ln(s)}.
  \end{align*}
  We apply Corollary~\ref{cor:newdrift} with $\tilde s_0 = n / \ln(n)^3 \le n/3$ and $\delta_{\max}(\tilde s_0) \le \frac{1}{e(r-1)n}(1 + O(1/n) + O(1/ \ln(n)^3 (r-1)))$ and obtain 
  \begin{align*}
  E[T] &\ge \frac{\ln(\tilde s_0) - \ln(s_{\aim})}{\delta_{\max}( \tilde s_0)} (1-2\beta)\\
         &\ge e(r-1)n\ln(n) - O((r-1)n\ln\ln(n)).
  \end{align*}
\end{proof}

We remark that the lower order term $O((r-1)n\log \log n)$ in this lower bound could be removed with stronger methods. We preferred to use the simple and natural proof approach via multiplicative drift, because it is easy to handle and still relatively precisely describes the true behavior of the process. As is visible from Lemma~\ref{lem:delta}, in the early stages the progress is slightly faster than the multiplicative (main) term $s / e (r-1) n$. This is why we cut out the regime from the initial $H$-value of approximately $n(1-1/r)$ up to an $H$-value of $\tilde s_0 = n/\ln(n)^3$, resulting in a $-\Theta((r-1) n \log\log n)$ term in our lower bound. Another $-\Theta((r-1) n \log \log n)$ term stems from the second condition of Witt's lower bound drift theorem (which is similar to the second condition of our theorem). To prove a bound sharp up to terms of order $(r-1) n \log \log n$, we need $\beta \le \log\log n / \log n$. However, this forbids using an $s_{\aim}$ smaller than $1/\beta = \log n / \log\log n$, since otherwise any improvement would count into the bad event of the second condition. An $s_{\aim}$ of at least polylogarithmic size immediately implies an $\Omega((r-1)n \log\log n)$ additive distance to the upper bound proven in Theorem~\ref{thm:eauniformupper}. We are very optimistic that via variable drift, in particular, the lower bound theorem of~\cite{DoerrFW11}, both difficulties could be overcome. We do not think that this small improvement justifies the effort, though.

\section{Unit Mutation Strength}
\label{sec:pm1}

In this section we regard the mutation operator that applies only $\pm 1$ changes to each component. 

It is not very surprising that RLS with the $\pm 1$ variation operator needs $\Theta(n(r + \log n))$ fitness evaluations in expectation to optimize any $r$-valued \OneMax function. We give the full proof below since it is similar to the analysis of the \oea equipped with the $\pm 1$ variation operator. The proof makes use of the following observation. There are two extreme kinds of individuals with fitness $n$. The first kind is only incorrect in one position (by an amount of $n$); the second kind is incorrect in every position (by an amount of $1$). The first kind of individual is hard to improve (the deficient position has to be chosen for variation), while the second kind is very easy to improve (every position allows for improvement). We reflect this in our choice of potential function by giving each position a weight exponential in the amount that it is incorrect, and then sum over all weights.

\begin{theorem}
\label{thm:RLS1}
The expected optimization time of RLS with the $\pm 1$ variation operator is $\Theta(n(r + \log n))$ for any $r$-valued \onemax function.
\end{theorem}

\begin{proof}
The lower bound $\Omega(nr)$ is quite immediate: with probability $1/2$ we start in a search point of fitness at most $nr/2$ and in each step the algorithm increases the fitness by at most one. On the other hand, there is a coupon collector effect which yields the $\Omega(n \log n)$ lower bound. Indeed, it is well-known that this is the expected number of RLS iterations that we need in case of $r=2$, and larger values of $r$ will only delay optimization.

We now turn to the more interesting upper bound. Let any $r$-valued \onemax function be given with metric $d$ and target $z$. We want to employ a multiplicative drift theorem (see Theorem~\ref{thm:multidrift}). We measure the potential of a search point by the following drift function. For all $x \in \Omega = [r]^n$, let
\begin{equation}
\label{def:g}
g(x):=\sum_{i=1}^n{(w^{d(z_i,x_i)}-1)},
\end{equation}
where $w:=1+\eps$ is an arbitrary constant between $1$ and $2$. In fact, for the analysis of RLS we can simply set $w:=2$ but since we want to re-use this part in the analysis of the \oea, we prefer the more general definition here.

We regard how the potential changes on average in one iteration. Let $x$ denote the current search point and let $y$ denote the search point that we obtain from $x$ after one iteration of RLS (after selection). 
Clearly, we have that each position is equally likely to be selected for variation. When a non-optimal component $i$ is selected, then the probability that $y_i$ is closer to $z_i$ than $x_i$ is at least $1/2$, while for every already optimized component we will not accept any move of RLS (thus implying $y_i=x_i$). This shows that, abbreviating $d_i:=d(z_i,x_i)$ for all $i\in [1..n]$, and denoting by $O:=\{i \in [1..n] \mid x_i=z_i\}$ the set of already optimized bits,
\begin{align*}
\E[g(x)-g(y)\mid x]
& = \tfrac{1}{2n}\sum_{i\in [1..n]\setminus O}{\left((w^{d_i}-1) - (w^{d_i-1}-1)\right)}\\
& =
\tfrac{1}{2n}\sum_{i\in [1..n]\setminus O}{(1-\tfrac{1}{w})w^{d_i}}\\
& \geq 
\tfrac{1}{2n} (1-\tfrac{1}{w}) \sum_{i \in [1..n]}{(w^{d_i}-1)}\\
& =
\tfrac{1}{2n} (1-\tfrac{1}{w}) g(x).
\end{align*}

Furthermore, the maximal potential that a search point can obtain is at most $n w^{r}$.
Plugging all this into the multiplicative drift (see Theorem~\ref{thm:multidrift}), we see that the expected optimization time is of order at most $\ln(n w^r)/\left(\tfrac{1}{2n} (1-\tfrac{1}{w})\right) = O(n (\log(n)+r))$, as desired.
\end{proof}

For the analysis of the \oea we will proceed similarly as for RLS. To help with the added complexity, we use the following lemma.

\begin{lemma}\label{lem:costOfSets}
Let $n$ be fixed, let $q$ be a cost function on \emph{elements} of $[1..n]$ and let $c$ be a cost function on \emph{subsets} of $[1..n]$. Furthermore, let a random variable $S$ ranging over subsets of $[1..n]$ be given. Then we have
\begin{equation}
\forall T \subseteq [1..n]: c(T) \leq \sum_{i \in T} q(i) \Rightarrow E[c(S)] \leq \sum_{i = 1}^n q(i) \Pr[i \in S];
\end{equation}
and
\begin{equation}
\forall T \subseteq [1..n]: c(T) \geq \sum_{i \in T} q(i) \Rightarrow E[c(S)] \geq \sum_{i = 1}^n q(i) \Pr[i \in S].
\end{equation}
\end{lemma}
\begin{proof}
We have $E[c(S)] = \sum_{T \subseteq [n]}\Pr[S=T]c(S) 
	\leq \sum_{T \subseteq [n]}\Pr[S=T] \sum_{i \in T}^n q(i)
	= \sum_{i = 1}^n q(i) \Pr[i \in S]$.
The other direction follows analogously.
\end{proof}

The proof for the case of the \oea follows along similar lines, but is (significantly) more involved.

\begin{theorem}\label{thm:oea1}
The expected optimization time of the \oea with the $\pm 1$ variation operator is $\Theta(n(r + \log n))$ for any $r$-valued \onemax function.
\end{theorem}

\begin{proof}
The lower bound $\Omega(nr)$ follows almost as for RLS: With constant probability the initial search point is $\Theta(nr)$ away from the optimum, and the expected progress towards the optimum is bounded form above by $1$. Thus, with a simple lower-bound additive drift theorem~\cite{HeY01}, the lower bound of $\Omega(nr)$ follows.

Regarding the upper bound, let any $r$-valued \onemax function be given with metric $d$ and target $z$. We want to employ multiplicative drift again. We fix some $w >1$ to be specified later. With any search point $x \in \searchSpace$ we associate a vector $d \in \realnum^n$ such that, for all $i \leq n$, $d_i = d(x_i,z_i)$. We use the same potential $g$ on $\searchSpace$ as for the analysis of RLS, that is, for all $x \in \searchSpace$,
$$
g(x) = \sum_{i=1}^n (w^{d_i} -1).
$$

Let any current search point $x \in \searchSpace$ be given and let $Y$ be the random variable describing the search point after one cycle of mutation and selection. Let $E_1$ be the event that $Y$ is obtained from $x$ by flipping exactly one bit and the result is accepted (that is, $f(Y) \le f(x)$). Let $E_2$ be the event that at least $2$ bits flip and the result is accepted. The total drift in the potential $g$ is now
$
E[g(x) - g(Y)] = E[g(x) - g(Y) \mid E_1]\Pr[E_1] + E[g(x) - g(Y) \mid E_2]\Pr[E_2].
$
We are now going to estimate $E[g(x) - g(Y) \mid E_2]$. The random variable $Y$ is completely determined by choosing a set $S \subseteq [1..n]$ of bit positions to change in $x$ and then, for each such position $i \in S$, choosing how to change it (away or towards the target $z_i$). For each choice $S \subseteq [1..n]$, let $A(S)$ be the set of all possible values for $Y$ which have the set $S$ as positions of change. For each possible $S \subseteq [1..n]$, let $Y(S)$ be the random variable $Y$ conditional on making changes exactly at the bit positions of $S$. 
Thus, we can now write the random variable $Y$ as$$
Y = \sum_{S} Y(S) \Pr[S].
$$
We are now going to estimate, for any possible $S$, 
$$
E[g(Y(S)) - g(x)].
$$
For all possible $S$, let $c(S) = E[g(Y(S)) - g(x)]$.
Let a possible $S$ be given. Note that $Y(S)$ is the uniform distribution on $A(S)$. 
For each $y \in A(S)$ and each $i \in S$, we have $d(y_i,z_i) - d_i \in \{-1,0,1\}$ (note that the case of being $0$ can only occur when $r$ is odd and we have a circle, or in other such border cases); in the case that this value is $1$ we call $(y,i)$ an \emph{up-pair}, and in case that this value is $-1$ we call this pair a \emph{down-pair}. We let $U$ be the set of all up-pairs. As we only consider accepted mutations, we have that, for all $y \in A(S)$, $\sum_{i \in S} d(y_i,z_i) - d_i \leq 0$. This implies that there are at least as many down-pairs as there are up-pairs in $A(S) \times S$. Furthermore, for any up-pair $(y,i)$ with $d_i \neq 0$ there is $y' \in A(S)$ such that $(y',i)$ is a down-pair and, for all $j \in S \setminus \{i\}$, $y'_j = y_j$. Thus, for all up-pairs $(y,i) \in U$ there is a down-pair $(\overline{y},\overline{i})$, such that the mapping $(y,i) \mapsto (\overline{y},\overline{i})$ is injective and, for all $(y,i) \in U$ with $d_i \neq 0$, $(\overline{y},\overline{i}) = (y',i)$.
Note that, for all up-pairs $(y,i)$, we have $d_i \leq d_{\overline{i}}$.

We now get, for any $(y,i) \in U$,
\begin{align*}
w^{d(y_i,z_i)} - w^{d_i} + w^{d(\overline{y}_{\overline{i}},z_{\overline{i}})} - w^{d_{\overline{i}}} 
& \leq
w^{d_i}(w - 1 + \frac{1}{w} - 1)\\ &= w^{d_i} \frac{(w-1)^2}{w}.
\end{align*}
Overall we have
\begin{eqnarray*}
c(S) 
 & = & E[g(Y(S)) - g(x)]\\
 & = & \frac{1}{|A(S)|}\sum_{y \in A(S)} g(y) - g(x)\\
 & = & \frac{1}{|A(S)|}\sum_{y \in A(S)} \sum_{i=1}^n \left( w^{d(y_i,z_i)} - w^{d_i} \right)\\
 & \leq & \frac{1}{|A(S)|}\sum_{(y,i) \in U} \left( w^{d(y_i,z_i)} - w^{d_i} + w^{d(\overline{y}_{\overline{i}},z_{\overline{i}})} - w^{d_{\overline{i}}}\right)\\
 & \leq & \frac{1}{|A(S)|}\sum_{(y,i) \in U} w^{d_i} \frac{(w-1)^2}{w}\\
 & \leq & \frac{1}{2} \sum_{i \in S} w^{d_i} \frac{(w-1)^2}{w}.
\end{eqnarray*}

Using Lemma~\ref{lem:costOfSets}, we see that
\begin{align*}
E[g(Y) - g(x) \mid E_2]
 & \leq \sum_{i=1}^n \frac{1}{n} w^{d_i} \frac{(w-1)^2}{2w}\\
 & = \frac{(w-1)^2}{2wn} \sum_{i=1}^n w^{d_i}.
\end{align*}

We use the following estimation of progress we can make with changing exactly one position.
\begin{align*}
E[g(x) - g(Y) \mid E_1]\Pr[E_1] & \geq \frac{1}{2ne}\sum_{i\in [n]}{(1-\tfrac{1}{w})w^{d_i}}\\
& = \frac{w-1}{2wne} \sum_{i=1}^n w^{d_i}.
\end{align*}

Let $w$ be any constant $>1$ such that $w-1 - e(w-1)^2 > 0$, and let $c = (w-1 - e(w-1)^2)/e$. Then we have
\begin{eqnarray*}
E[g(x) - g(Y) ]
 & \geq & \frac{w-1}{2wne}\sum_{i = 1}^n{w^{d_i}} - \frac{(w-1)^2}{2wn} \sum_{i=1}^n w^{d_i}\\
 & = & \frac{w-1 - e(w-1)^2}{2wne}\sum_{i = 1}^n{w^{d_i}}\\
 & = & \frac{c}{2wn}\sum_{i=1}^n{w^{d_i}}\\
 & \geq & \frac{c}{2wn} g(x).
\end{eqnarray*}

Again, the maximal potential that a search point can obtain is at most $n w^{r}$.
Plugging all this into the multiplicative drift (see Theorem~\ref{thm:multidrift}), we see that the expected optimization time is of order at most $\ln(n w^r)/\left(\tfrac{c}{2wn}\right) = O(n (\log(n)+r))$, as desired.
\end{proof}

\section{Harmonic Mutation Strength}
\label{sec:HarmonicStepSize}

In this section we will consider a mutation operator with variable step size. The idea is that different distances to the target value require different step sizes for rapid progress. We consider a mutation operator which, in each iteration, chooses its step size from a fixed distribution. As distribution we use what we call the \emph{harmonic distribution}, which chooses step size $j \in [1..r-1]$ with probability proportional to $1/j$. Using the bound on the harmonic number $H_{r-1} < 1 + \ln r$, we see that the probability of choosing such a $j$ is at least $1/(j (1 +\ln r))$.

\begin{theorem}\label{thm:HarmonicDistribution}
The RLS as well as the \oea with the harmonically distributed step size (described above) has an expected optimization time of $\Theta(n \log r (\log n + \log r))$ on any $r$-valued \OneMax function.
\end{theorem}
\begin{proof}
We first show the upper bound by considering drift on the fitness. Let any $x \in \searchSpace$ be given, let $Y$ be the random variable describing the best individual of the next iteration and let $A_{i,j}$ be the event that $Y$ differs from $x$ in exactly bit position $i$ and this bit position is now $j$ closer to the optimum. Note that, for both RLS and the \oea, we get $\Pr[A_{i,j}] \geq \frac{1}{2en j (1 +\ln r)}$. We have
\begin{align*}
\E[f(x) &- f(Y)]
  \geq  \sum_{i=1}^n \sum_{j=1}^{d_i} \E[f(x) - f(Y) \mid A_{i,j}]\Pr[A_{i,j}]\\
 & =     \sum_{i=1}^n \sum_{j=1}^{d_i} j \Pr[A_{i,j}]
  \geq  \sum_{i=1}^n \sum_{j=1}^{d_i} \frac{j}{2en j (1 +\ln r)}\\
 & =     \sum_{i=1}^n \frac{d_i}{2en (1 +\ln r)}
  =     \frac{1}{2en (1 +\ln r)} f(x).
\end{align*}
As the initial fitness is less than $rn$, the multiplicative drift theorem (see Theorem~\ref{thm:multidrift}) gives us the desired total optimization time.

Now we turn to the lower bound. A straightforward coupon collector argument gives us the lower bound of $\Omega(n \log r \log n)$, since each position has to change from incorrect to correct at some point, and that mutation has a probability of $O(1/(n\log r))$. It remains to show a lower bound of $\Omega(n (\log r)^2)$. To this end, let $f$ be any $r$-values \onemax function and $x^*$ its optimum. Let $g(x) = d(x_1,x^*_1)$ be the distance of the first position to the optimal value in the first position. Let $h(x) = \ln(g(x)+1)$. Let $x'$ be the outcome of one mutation step and $x''$ be the outcome of selection from $\{x,x'\}$. We easily compute $E[\max\{0,h(x)-h(x')\}] \le \frac{K}{n \ln r}$ for some absolute constant $K$. Consequently, $E[h(x) - h(x'')] \le \frac{K}{n \ln r}$ as well. For the random initial search point, we have $g(x) \ge r/2$ with constant probability, that is, $h(x) = \Omega(\log r)$ with constant probability. Consequently, the additive drift theorem gives that the first time $T$ at which $h(x)=0$, satisfies $E[T] \ge \Omega(\log r) / \frac{K}{n \ln r} = \Omega(n \log^2 r)$. 
\end{proof}

In the same way as we showed the additive drift statement $E[h(x) - h(x'')] = O(1 / n \log r)$, we could have shown a multiplicative drift statement for $g$, namely $E[g(x) - g(x'')] = O(g(x) / n \log r)$; in fact, the latter is implied by the former. Unfortunately, due to the presence of large jumps -- we have $\Pr[g(x'') \le g(x)/2] = \Theta(1 / n \log r)$ --, we cannot exploit this via the lower bound multiplicative drift theorem.

Naturally, the question arises whether the $O((\log r)^2)$ dependence on $r$ can be improved. In particular, one wonders whether drawing the step size from the Harmonic distribution is optimal, or whether another distribution gives a better optimization time. This is exactly the problem considered in \cite{Dit-Row-Weg-Woe:j:10}, where the following result is presented, which could also be used to derive the run time bound of Theorem~\ref{thm:HarmonicDistribution}.

\begin{theorem}[{\cite{Dit-Row-Weg-Woe:j:10}}]
\label{thm:DIT}
Let a random process on $A = \{0, \ldots, r\}$ be given, representing the movement of a token. Fix a probability distribution of step sizes $D$ over $\{1, \ldots, r\}$. Initially, the token is placed on a random position in $A$. In round $t$, a random step size $d$ is chosen according to $D$. If the token is in position $x \geq d$, then it is moved to position $x - d$, otherwise it stays put. Let $T_D$ be the number of rounds until the token reaches position $0$. Then $\min_D (E[T_D]) = \Theta((\log r)^2)$.
\end{theorem}

While our processes have a slightly different behavior (including the possibility to overshoot the goal), we believe that these differences only lead to to differences in the constants of the optimization time. Thus, the above theorem indicates that the Harmonic distribution is an optimal choice and cannot be improved.

\section{Conclusion}

While many analyses of randomized search heuristics focus on the behavior of the algorithm in dependence of a large and growing dimension, we additionally considered a growing size of the search space in each dimension. We considered the \oea with three different mutation strengths and proved asymptotically tight optimization times for a variety of \onemax-type test functions over an alphabet of size~$r$. We proved that both using large changes (change to uniformly chosen different value) or very local changes (change value by $\pm 1$) leads to relatively slow (essentially linear in $r$) optimization times of $\Theta(rn\log n)$ and $\Theta(n (r + \log n))$, respectively. 

We then considered a variable step size operator which allows for both large and small steps with reasonable probability; this leads to an optimization time of $\Theta(n\log r(\log n + \log r))$. Note that this bound, while polylogarithmic in $r$, is worse than the bound of $\Theta(n (r + \log n))$ for the $\pm 1$ operator when $r$ is asymptotically smaller than $\log n \log\log n$. This shows that there is no uniform superior mutation operator among the three proposed operators.

\subsection*{Acknowledgments}
This research benefited from the support 
of the ``FMJH Program Gaspard Monge in optimization and operation research'', 
and from the support to this program from EDF (\'Electricit\'e de France).

\end{document}